\documentclass[11pt]{article}

\usepackage{latexsym,mathrsfs}
\usepackage{amsmath,amssymb} 
\usepackage{amsthm,enumerate,verbatim}
\usepackage{amsfonts}
\usepackage{graphicx}
\usepackage{algorithm}
\usepackage{algorithmic}
\usepackage{url}
\usepackage{color}

\newtheorem{lemma}{Lemma}

\newtheorem{theorem}{Theorem}
\newtheorem{corollary}{Corollary}
\newtheorem{definition}{Definition}
\newtheorem*{definition*}{Definition}

\DeclareMathOperator{\argmin}{argmin} 
\DeclareMathOperator{\cone}{cone}

\DeclareMathOperator{\logdet}{logdet}

\definecolor{brightpink}{rgb}{1.0, 0.0, 0.5}

\title{Checking the Sufficiently Scattered Condition \\ 
using a Global Non-Convex Optimization Software} 
\date{}

\author{
Nicolas Gillis\thanks{Email: \{nicolas.gillis\}@umons.ac.be. NG acknowledges the support by the European Union (ERC consolidator, eLinoR, no 101085607), and by the Francqui Foundation.} \\ University of Mons \\ Mons, Belgium 
\and Robert Luce \\ Gurobi 
	}

\begin{document}

\maketitle

\begin{abstract}
The sufficiently scattered condition (SSC) is a key condition in the study of identifiability of various matrix factorization problems, including 
nonnegative, 
minimum-volume, 
symmetric, 
simplex-structured, 
and polytopic matrix factorizations. The SSC allows one to guarantee that the computed matrix factorization is unique/identifiable, up to trivial ambiguities.  
However, this condition is NP-hard to check in general. In this paper, we show that it can however be checked in a reasonable amount of time in realistic scenarios, when the factorization rank is not too large. 
This is achieved by formulating the problem as a non-convex quadratic optimization problem over a bounded set. We use the  global non-convex optimization software Gurobi, and showcase the usefulness of this code on synthetic data sets and on real-world hyperspectral images.   
\end{abstract}


\section{Introduction} \label{sec:intro}

Low-rank matrix factorizations (LRMFs) are central techniques in numerical linear algebra, with the singular value decompositions (SVD) and principal component analysis (PCA) as the most famous examples. 
LRMFs are widely used in data analysis, statistics, signal processing, control, optimization, and machine learning; see, e.g.,~\cite{markovsky2011low, UHZB14, kishore2017literature}. 
Given an input matrix $X \in \mathbb{R}^{m \times n}$ and a factorization rank $r$, LRMF aims at finding 
$W \in \mathcal{W} \subseteq \mathbb{R}^{m \times r}$ and 
$H \in \mathcal{H} \subseteq \mathbb{R}^{r \times n}$ such that $X \approx WH$. The sets  $\mathcal{W}$ and $\mathcal{H}$ impose constraints on $W$ and $H$, such as orthogonality in SVD and PCA, and sparsity in sparse PCA~\cite{dAspremont2007spca}. 
These additional constraints allow one to more easily interpret the factors, and are often motivated by the application at hand. 

Among LRMFs, nonnegative matrix factorization (NMF)~\cite{lee1999learning}, which imposes nonnegative constraint on the factors and assumes $X$ is nonnegative, has become a standard tool as well; see~\cite{cichocki2009nonnegative, xiao2019uniq, gillis2020nonnegative} and the references therein. 
Nonnegativity is motivated for example by physical considerations, e.g., in imaging, audio signal processing and chemometrics, or by probabilistic interpretations, e.g., in topic modeling.  
A key aspect in these applications is that the factorization is unique, a.k.a.\ identifiable (we will use both terms interchangeably), 
which allows one to recover the groundtruth factors that generated the data (such as the sources in blind source separation).  
A factorization, $X = WH$, is unique/identifiable  if for any other factorization, $X = W'H'$, there exists a permutation $\pi$ of $\{1,2,\dots,r\}$ and scaling factors $\{\alpha_k\}_{k=1}^r$ such that for all $k$ 
\begin{equation} \label{eq:permscal}
W'(:,k) = \alpha_k W(:,\pi_k) 
\quad \text{ and } \quad 
H'(k,:) = \alpha_k^{-1} H(\pi_k,:). 
\end{equation}
Unfortunately, nonnegativity is typically not enough to ensure the uniqueness; 
see the discussions in~\cite{fu2018identifiability, gillis2020nonnegative} and the references therein.  
A stronger condition that ensures uniqueness is the so-called sufficiently scattered condition (SSC) which requires some degree of sparsity within the factors $W$ and $H$; see Section~\ref{sec:ssc} for a formal definition. 
The uniqueness of NMF under the SSC was presented in~\cite{huang2013non}, and later lead to numerous identifiability results for other LRMFs, namely: 
\begin{itemize}
    \item Minimum-volumne NMF~\cite{fu2015blind, lin2015identifiability} which seeks for an NMF decomposition that minimizes the volume of the convex full of the columns of $W$. 

    \item Simplex-structured matrix factorization~\cite{lin2018maximum, abdolali2021simplex} which relaxes the constraints on $W$ but imposes $H$ to be column-wise stochastic. A bounded version, where the entries of $W$ are bounded~\cite{vuthanh2023bounded}, has also been explored and shown to be identifiable under SSC-like conditions. 
    
    \item Symmetric NMF with applications in topic modeling~\cite{fu2016robust, fu2018anchor}. 

\end{itemize}

The SSC condition was also generalized for polytopic matrix factorizations~\cite{tatli2021polytopic, vuthanh2023bounded} and for bounded component analysis~\cite{hu2023identifiable} which are  generalizations of NMF where the nonnegative orthant is replaced by polytopes. 
The SSC has also been used to provide identifiability in many other contexts where constrained matrix factorizations play a crucial role. 
This has been the case in particular in machine learning tasks, such as 
topic modeling~\cite{huang2016anchor, fu2018anchor}, 
crowd sourcing~\cite{ibrahim2019crowdsourcing}, 
recovering joint probability~\cite{ibrahim2021recovering}, 
label-noise learning~\cite{li2021provably}, deep constrained clustering~\cite{nguyen2023deep}, 
 dictionary learning~\cite{hu2023dico}, and tensor decompositions~\cite{sun2023volume}.  


\paragraph{Outline and contribution} In summary, the SSC plays a critical role in checking whether a wide class of LRMFs with nonnegativity constraints are identifiable. Unfortunately, the SSC is NP-hard to check in general (see Section~\ref{sec:ssc} for more details). 
To the best of our knowledge, there currently does not exist a solver that checks the SSC, even for small-size problems. 
In this paper, we overcome this limitation by leveraging current global non-convex quadratic optimization software, and in particular Gurobi, to check the SSC for relatively large matrices. 

The paper is organized as follows. We first define the SSC rigorously, in Section~\ref{sec:ssc}, after having introduced useful concepts in convex geometry. Then we show in Section~\ref{sec:formu} how checking the SSC is equivalent to solving an non-convex quadratic program. 
In Section~\ref{sec:solvformu}, we provide an alternative formulation with box constrained, which is crucial to using global non-convex quadratic optimization software. 
We report numerical experiments in Section~\ref{sec:numexp} on synthetic and real data sets, showing that Gurobi can solve relatively large instance, for a factorization rank $r$ up to a few dozen, and input matrices of size up to a few thousands.

\paragraph{Notation} Given a vector $x \in \mathbb{R}^r$, $\|x\|_2$ 
denotes its $\ell_2$ norm and $x^\top$ its transpose. The nonnegative orthant in dimenion $r$ is denoted  $\mathbb{R}^r_+$. 
The vector of all ones is denoted $e$, and of all zeros $0$. 
The identity matrix is denoted $I$, and its $i$th column is denoted $e_i$ (a.k.a.\ the $i$th unit vector).  
The dimensions of $e$, $0$, $I$ and $e_i$ will be clear from the context. For a matrix $H$, $H(:,j)$ and $H(i,:)$ 
denote its $j$th column and $i$th row, respectively. The set $\{1,2,\dots,r\}$ is denoted $[r]$.

\section{Preliminaries: 
cones, their duals and the SSC} \label{sec:ssc}

Given a matrix $H \in \mathbb{R}^{r \times n}$, we define the cone generated by its columns as 
\[
\cone(H) = \{ x \ | \ x = Hy, y \geq 0\}. 
\] 
The dual of a cone $\mathcal{H}$ is defined as 
\[
\mathcal{H}^* = \big\{ x \ | \ x^\top z \geq 0 \text{ for all } z \in \mathcal{H} \big\}. 
\]
In particular, the dual cone of $\cone(H)$ is given by 
\begin{align*}
\cone^*(H) 
& = \big\{ x \ | \ x^\top z \geq 0, \text{ for all } z \in \cone(H) \big\} \\ 
& = \big\{ x \ | \ x^\top H y = (H^\top x)^\top y \geq 0, \text{ for all } y \geq 0 \big\} \\
& = \big\{ x \ | \ H^\top x \geq 0 \big\}. 
\end{align*}
Another cone we will need is the following second-order cone:  
\[ 
\mathcal{C} 
= \big\{ x \in \mathbb{R}^r  
\ | \ e^\top x \geq \sqrt{r-1} \| x \|_2 \big\}, 
\] 
which is contained in the nonnegative orthant. 
Its dual cone is given by 
$\mathcal{C}^* 
= \big\{  x \in \mathbb{R}^r  
\ | \ e^\top x \geq \| x \|_2 \big\}$ and contains the nonnegative orthant (which is self-dual). 
An important and easy-to-check property of duality is that $\mathcal{H}_1 \subseteq \mathcal{H}_2$ if and only if 
$\mathcal{H}_2^* \subseteq \mathcal{H}_1^*$.

\paragraph{The SSC} We will use the following definition of the SSC. 
\begin{definition}[SSC, \cite{huang2013non}] \label{def:SSC}
A nonnegative matrix $H \in \mathbb{R}^{r \times n}_+$ satisfies the sufficiently scattered condition (SSC)  
if 
\begin{equation} \label{eq:SSCmain}
\mathcal{C} 
= \{ x \in \mathbb{R}^r  
\ | \ e^\top x \geq \sqrt{r-1} \| x \|_2 \} 
\quad 
\subseteq 
\quad  
\cone(H) = \{ x \ | \ x = Hy, y \geq 0\}, \tag{SSC1}
\end{equation} 
and 
\begin{equation} \label{eq:SSCcond2}
\text{any $q \in \cone^*(H) \cap \big\{x \ | \ e^\top x = \|x\|_2 \big\}$ is a scaling of a unit vector (that is, of $e_i$ for some $i$).}   \tag{SSC2} 
\end{equation}
\end{definition} 

There actually exist several slight variations of the definition of the SSC.  
All of them include the requirement~\eqref{eq:SSCmain}, while the second condition is slightly modified: 
\begin{itemize}
    \item \cite{fu2015blind} requires that there does not exist any orthogonal matrix $Q$ such that $\cone(H) \subseteq \cone(Q)$, except for permutation matrices. This is a slight relaxation of~\eqref{eq:SSCcond2}. 

    \item \cite{lin2015identifiability} requires $\cone(H)$ to contain a slightly larger cone than $\mathcal{C}$, namely, $\mathcal{C}_q = \{ x \in \mathbb{R}^r_+ 
\ | \ e^\top x \geq q \| x \|_2 \}$ for any $q < \sqrt{r-1}$, which is slightly more restrictive 
than~\eqref{eq:SSCcond2}.   
\end{itemize}

We chose Definition~\ref{def:SSC} because it is slightly simpler to present; however, our formulations can be  easily adapted to handle the other definitions above.

\paragraph{Identifiability} 
As mentioned in the introduction, the SSC allows one to provide identifiability results for various matrix factorizations with nonnegativity constraints. It is out of the scope of this paper to list all of these results, and we just mention the first that appeared in the literature, for NMF. We discuss another one, namely minimum-volume NMF, in Section~\ref{sec:hsi}. 
\begin{theorem}[\cite{huang2013non}] \label{th:uniqNMF}
    Let $X = W H$ be an NMF of $X$ with factorization rank $r$, where ${W}^\top$ and $H$ satisfy the SSC. 
    Then this NMF is unique, that is, any other NMF of $X$ with factorization rank $r$, $X = W'H'$ with $W' \geq 0$ and $H' \geq 0$, can be obtained by permutation and scaling of $WH$; see~\eqref{eq:permscal}. 
\end{theorem}

\section{Checking the SSC: necessary condition 
 and formulation} \label{sec:formu} 

Before checking the SSC, it will be useful to check whether a simple necessary condition holds. 

\paragraph{Necessary condition for the SSC} 

The cone $\mathcal{C}$ contains the points $e-e_i$ for $i \in [r]$ at its border, since $e^\top(e-e_i) = \sqrt{r-1}\|e-e_i\|_2 = r-1$, and hence 
\[
\mathcal{T} = \cone\big(ee^\top - I\big)  
\; \subset \; \mathcal{C}. 
\] 
We therefore have the following necessary condition for the SSC. 
\begin{definition}[Necessary Condition for the SSC, NC-SSC]\cite[p.~119]{gillis2020nonnegative} 
    The matrix $H \in \mathbb{R}^{r \times n}_+$ satisfies the NC-SSC if  $\mathcal{T} = \cone\big(ee^\top - I\big) \subset \cone(H)$, that is, $e-e_i \in \cone(H)$ for $i \in [r]$. 
\end{definition}

The NC-SSC can be easily checked, in polynomial time, by solving systems of linear inequalities: for all $i \in [r]$, check that there exists $y \geq 0$ such that $e-e_i = Hy$. 

In the remainder, we will assume that this necessary condition has been checked, that is, $\mathcal{T} \subset \cone(H)$, otherwise $H$ cannot satisfy the SSC. 

Note that the NC-SSC (and hence the SSC) require a certain degree of sparsity of $H$, since its cone must contain the vectors $\{e-e_i\}_{i=1}^r$ that contain a zero entry. In fact, one can show that a necessary condition for the SSC to hold is that $H$ has at least $r-1$ zeros per row~\cite{huang2013non, fu2018identifiability, gillis2020nonnegative}.

\paragraph{Checking the SSC via non-convex quadratic optimization}   

The first condition of the SSC~\eqref{eq:SSCmain} is equivalent to $\cone^*(H) \subseteq \mathcal{C}^*$. 
This condition is not satisfied if there exists $x \in \cone^*(H)$ while $x \notin \mathcal{C}^*$, that is, if there exists $x$ such that 
\begin{equation} \label{eq:checkSSC1}
H^\top x \geq 0  \quad \text{ and } \quad e^\top x < \|x\|_2. 
\end{equation}
These conditions remain valid when $x$ is multiplied by any positive constant. 
Moreover, we have the following lemma. 
\begin{lemma} \label{lem:etxpos}
    Let $H \in \mathbb{R}^{r \times n}_+$ satisfy the NC-SSC. 
    Then for any $x \neq 0$ satisfying $H^\top x \geq 0$, 
    we have $e^\top x > 0$. 
\end{lemma}
\begin{proof} 
Since $H$ satisfies the NC-SSC, $\mathcal{T} = \cone(ee^\top - I) \subset \cone(H)$, 
and hence $\cone^*(H) \subset \mathcal{T}^*$, where 
\begin{equation*} 
\mathcal{T}^*   
= \Big\{ x \in \mathbb{R}^r \ | \ (ee^\top - I) x \geq 0 \Big\}  
  = \Big\{x \in \mathbb{R}^r \ \big| \ e^\top x - x_i = \sum_{i \neq j} x_i \geq 0 \text{ for any } j \in [r] \Big\}. 
\end{equation*}  
Summing the inequalities defining $\mathcal{T}^*$, we obtain $e^\top x \geq 0$. 
    It therefore remains to show that, for $x \in \mathcal{T}^*$, 
    $e^\top x = 0$  if and only if $x = 0$. 
    In fact, $e^\top x = 0$ implies that all inequalities defining $\mathcal{T}^*$ must be active at $x$, otherwise their sum is positive and we would get $e^\top x > 0$, a contradiction. 
   Hence $(ee^\top - I) x = 0$ implying $x = 0$ since $ee^\top - I$ is full rank.  
\end{proof}


Now, going back to~\eqref{eq:checkSSC1}, we have the following corollary. 
\begin{corollary} \label{cor:etx1}
    Let $H \in \mathbb{R}^{r \times n}_+$ satisfy the NC-SSC. 
    Then~\eqref{eq:SSCmain} is not satisfied if and only if the system 
\begin{equation} \label{eq:SSCmainetx1} 
H^\top x \geq 0 \quad \text{ and } \quad e^\top x = 1 < \|x\|_2 
\end{equation} 
has a solution. 
\end{corollary}
\begin{proof}
Since $x = 0$ does not satisfy $e^\top x < \|x\|_2$, we can assume w.l.o.g., by Lemma~\ref{lem:etxpos},  that $e^\top x > 0$, and hence, using the scaling degree of freedom, that $e^\top x = 1$. 
\end{proof}

Checking whether~\eqref{eq:SSCmainetx1} has a solution for $H$ satisfying the NC-SSC can be done by solving the following optimization problem: 
\begin{equation}
p^* \quad   = \quad   \max_{x} \|x\|_2 
\; \text{ such that } \;  
e^\top x = 1 \text{ and } H^\top x \geq 0. \label{eq:SSCopt}
\end{equation} 
If $p^* > 1$, then $H$ does not satisfy~\eqref{eq:SSCmain}. 
Otherwise, $p^* = 1$, and to check whether $H$ satisfies the SSC, 
we need to check the second condition~\eqref{eq:SSCcond2}, that is, whether there exists $q \neq e_i$ for all $i$ 
such that $e^\top q = \|q\|_2$ and $H^\top q \geq 0$. \\ 

In summary, $H$ satisfies the SSC if and only if 
(i)~$H$ satisfies the NC-SSC, 
and (ii)~the only optimal solutions of~\eqref{eq:SSCopt} are 
$e_i$ for $i \in [r]$, with optimal value $p^* = 1$. This is why it is NP-hard to check the SSC, because~\eqref{eq:SSCopt} is the maximization of a convex function over a polytope which is NP-hard in general~\cite{freund1985complexity, huang2013non}.

\section{Solving~\eqref{eq:SSCopt} with Global Non-Convex Quadratic Optimization} \label{sec:solvformu}

The problem~\eqref{eq:SSCopt} cannot directly be handled by global non-convex  quadratic optimization solvers, such as Gurobi, because such solvers require a bounded feasible set, with lower and upper bounds on the variables. This allows them to use the McCormick envelopes~\cite{mccormick1976computability}, and rely on branch-and-bound strategies.  
In a few words, the idea is as follows: for every product of two variables that appears in the objective or in the constraints, say the product of the variables $x$ and $y$, a new variable is introduced, $z = xy$. Given that $x$ and $y$ belong to a bounded box, that is, $x \in [\underline{x}, \bar{x}]$ and $y \in [\underline{y}, \bar{y}]$, the equality $z=xy$ is approximated from above and below with linear constraints as follows: 
\begin{equation*}
 z  \leq \underline{y} x + \bar{x} y - \bar{x}  \underline{y},  \; 
 z  \leq  \bar{y} x + \underline{x} y - \underline{x} \bar{y},  \;
 z  \geq \underline{y} x +\underline{x} y -\underline{x} \underline{y}, \; 
 z  \geq \bar{y} x       + \bar{x} y - \bar{x} \bar{y}. 
\end{equation*}
 Figure~\ref{fig:mccormick} provides an illustration of such a McCormick envelope. 
 \begin{figure}[ht!]
\begin{center}
\includegraphics[width=10cm]{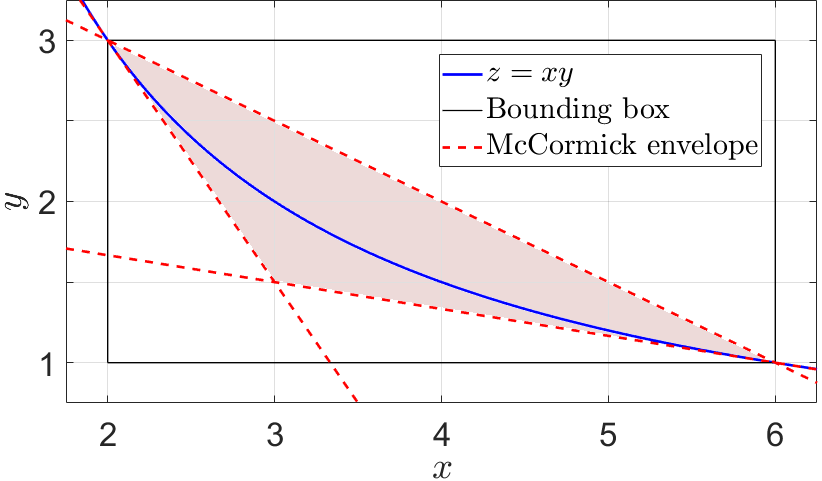} 
\caption{Illustration of the McCormick envelope for the nonlinear constraint $xy = 6$, with $x \in [2,6]$ and $y \in [1,3]$. (Note that, in this example, the last two inequalities, defining the upper bound, coincide.)}
\label{fig:mccormick}
\end{center}
\end{figure} 
This McCormick linearization is then improved along the algorithm by splitting the feasible set in smaller and smaller intervals for each variable, following branch-and-bound strategies; see, e.g., \cite{mitchell2014convex} and the references therein.    

Let us now show how any feasible solution $x$ of~\eqref{eq:SSCopt} can be bounded. 
\begin{lemma} \label{lem:bounds}
Let $H \in \mathbb{R}^{r \times n}_+$ satisfy the NC-SSC. 
Then the entries of any feasible solution $x$ of~\eqref{eq:SSCopt} satisfy $x_i \in [2-r,1]$ for all $i$. 
\end{lemma}
\begin{proof}
Since $\cone(ee^\top - I) \in \cone(H)$ for $i \in [r]$, 
we have $(ee^\top - I) x = e - x \geq 0$ and hence $x \leq e$. 
Since $e^\top x = 1$ and $x \leq e$, we have $x_i = 1 - \sum_{j \neq i} x_j \geq 1 - (r-1) = 2-r$ for all $i$. 
\end{proof}
Note that the lower bound in Lemma~\ref{lem:bounds}, $2-r$, can be achieved: this is the case for $H=(ee^\top - I)$ for which one can check that optimal solutions of~\eqref{eq:SSCopt} are given by $e + (1-r) e_i$ for $i \in [r]$.

\paragraph{Tightening the lower bound}

When it comes to checking the SSC, we now show that we can actually tighten the constraint $x_i \geq 2-r$ to $x_i \geq -1$. 
The reason is that we are not interested in the exact optimal value of~\eqref{eq:SSCopt}, but only to know whether it is larger than one. 
\begin{lemma} \label{lem:reform}
Let $H \in \mathbb{R}^{r \times n}_+$ satisfy the NC-SSC.  
Then the optimal value of~\eqref{eq:SSCopt} is equal to one, that is, $p^* = 1$, 
if and only of the optimal value  
\begin{equation}
q^* \quad = \quad \max_{x} \|x\|_2 
 \; \text{ such that } \;  
e^\top x = 1,  H^\top x \geq 0, \text{ and } 
-1 \leq x_i \leq 1 \text{ for all } i, \label{eq:SSCoptBND}
\end{equation} 
is equal to one. 
\end{lemma}
\begin{proof}
First note that, for $r \leq 3$, Lemma~\ref{lem:bounds} provides the result, since $x_i \geq 2-r \geq -1$. Hence we can assume $r \geq 4$.

We have $p^* \geq q^*$ since~\eqref{eq:SSCopt} is a relaxation of \eqref{eq:SSCoptBND}, 
and that $q^* \geq 1$ since $x=e_i$ for all $i$ are feasible solutions with objective equal to one. 

It remains to show that $p^* > 1$ implies $q^* > 1$. 
Let $x^*$ be an optimal solution of~\eqref{eq:SSCopt} with $p^* > 1$. 
By Lemma~\ref{lem:bounds}, $x^* \leq e$.  
If $x^*_i \geq -1$ for all $i$, we are done since $x^*$ is feasible for \eqref{eq:SSCoptBND} and $q^* = p^* > 1$. 
Otherwise assume $x_i^* < -1$ for some $i$. 
Let us consider the solution $y = \lambda  x^* + (1-\lambda) e_i$ for some $\lambda \in [0,1]$ (to be chosen below).  
    By convexity of the feasible set, $y$ is feasible for~\eqref{eq:SSCopt}. Let us take $\lambda$ such that $y_i = -1$, that is, 
    \[
\lambda = \frac{2}{1-x_i} \; \in \;  
\left[\frac{2}{r-1},1\right),   
    \] 
    so that $\lambda \in (0,1)$ 
    since $2-r \leq x_i < -1$ (Lemma~\ref{lem:bounds}) and $r \geq 4$. 
    The new solution $y$ satisfies $y_i = -1$, while the other entries of $y$ are equal to that of $x^*$ multiplied by $\lambda \in (0,1)$, and hence their absolute value gets smaller. However, we have $\|y\|_2 > 1$ since $y_i = -1$ while $e^\top y = 1$. (In fact\footnote{This follows from the inequalities 
    $\sqrt{r-1} \sqrt{\sum_{j\neq i} |y_j|^2} \geq \sum_{j\neq i} |y_j| 
    \geq \sum_{j\neq i} y_j > 1$ which implies $\sum_{j\neq i} |y_j|^2 > \frac{1}{r-1}$.}, we have $\|y\|_2^2 > 1 + \frac{1}{r-1}$.) 
    Hence, we have constructed a feasible solution $y$ of~\eqref{eq:SSCopt} such that $y_i = -1$ while $\|y\|_2 > 1$. 
For any other entry of $y$ smaller than $-1$, we can apply the same trick as for $x^*$, and we will eventually get a feasible solution of~\eqref{eq:SSCoptBND} with at least one entry equal to $-1$, and hence an objective function value strictly larger than 1. 
\end{proof}

In summary, we have the following theorem. 
\begin{theorem} \label{th:mainTH}
    The matrix $H \in \mathbb{R}^{r \times n}_+$ satisfies the SSC if and only if the following two conditions are satisfied  
    \begin{enumerate}
        \item The matrix $H$ satisfies the NC-SSC, that is, $e-e_i \in \cone(H)$ for all $i \in [r]$. 

        \item The optimal value of~\eqref{eq:SSCoptBND} is equal to $q^* = 1$, 
        and 
        the set of optimal solutions is
        $\{e_i\}_{i=1}^r$. 
    \end{enumerate}
\end{theorem}
\begin{proof}
    This follows from 
    Corollary~\ref{cor:etx1} and Lemma~\ref{lem:reform}. 
\end{proof}

\paragraph{Gurobi: Getting more than one solution, early stopping, and time limit} 

To check the SSC, we therefore have to check the NC-SSC, and then solve~\eqref{eq:SSCoptBND}. If the optimal value is equal to one, we need to check whether there exist optimal solutions different from 
$\{e_i\}_{i=1}^r$. 
To solve~\eqref{eq:SSCoptBND}, we rely on the global non-convex optimization software Gurobi\footnote{\url{https://www.gurobi.com/solutions/gurobi-optimizer/}}. 
Luckily, Gurobi allows one to generate more than one solution. Hence, it suffices to ask Gurobi to provide at least $r+1$ solutions using the following paramters:  \texttt{params.PoolSearchMode = 2; params.PoolSolutions = r+1}.  



Moreover, since we only care about checking whether the optimal value of~\eqref{eq:SSCoptBND} is larger than one, we can stop Gurobi as soon as the best found solution has objective strictly larger than 1, and we use the parameter: \texttt{params.BestObjStop = 1.0001}. This stopping criterion is key, as it allows Gurobi to stop very early for matrices far from satisfying the SSC.   

Finally, we do not know in advance how long it will take for Gurobi to solve~\eqref{eq:SSCoptBND}, and hence it is useful to use a timelimit (e.g., 5 minutes). If Gurobi cannot finish within 5 minutes, it means it was not able to find a solution with objective 
larger than 1. 
Hence, on top of the NC-SSC being satisfied (this is the first step of the algorithm), there has been additional necessary conditions satisfied, that is, all nodes explored within the branch-and-bound strategy do not violate the system~\eqref{eq:SSCmainetx1}. Hence although we cannot guarantee the SSC, the chances the SSC is satisfied are high. \\

Algorithm~\ref{algo:checkSSC} summarizes our algorithm to check the SSC. To the best of our knowledge, it is the first algorithm to exactly check the SSC, up to machine precision. 
\algsetup{indent=2em}
\begin{algorithm}[ht!]
\caption{Checking the SSC \label{algo:checkSSC}}
\begin{algorithmic}[1] 
\REQUIRE A nonnegative matrix $H \in \mathbb{R}^{r \times n}_+$. 

\ENSURE SSC $= 1$ is $H$ satisfies the SSC (Definition~\ref{def:SSC}), SSC $= 0$ otherwise. 
    \medskip  

	\STATE \emph{\% First step: Check the NC-SSC.} 
\FOR{$i$ = 1 : $r$} 
  \IF{there does not exist a feasible solution to the system 
  $e-e_i = Hx, x \geq 0$}
  \STATE SSC = 0; return. 
  \ENDIF 
\ENDFOR

\STATE \emph{\% Second step: Check the SSC.} 

\STATE Solve 
\[
q^* \quad = \quad \max_{x} \|x\|_2 
 \; \text{ such that } \;  
e^\top x = 1,  H^\top x \geq 0, \text{ and } 
-1 \leq x_i \leq 1 \text{ for all } i. 
\]
\IF{$q^* > 1$ \textbf{or} there is an optimal solution $x^* \neq e_i$ for $i\in [r]$}
\STATE SSC $= 0$. 
\ELSE
\STATE SSC $= 1$.
  \ENDIF
 
\end{algorithmic}  
\end{algorithm}


\section{Numerical experiments}  \label{sec:numexp}

In~\cite{fu2018identifiability}, authors use a heuristic to solve~\eqref{eq:SSCoptBND}, while, in~\cite[Chapter 4.2.3.6]{gillis2020nonnegative}, the author relied on the necessary condition that the vectors $\{e-e_i\}_{i=1}^r$ belong to the relative interior of $\cone(H)$. For the first time, we will provide results where the SSC is checked exactly. 
We first provide results on synthetic data, and then on real hyperspectral images factorized with minimum-volume NMF. 

All experiments are performed with a 
12th Gen Intel(R) Core(TM) i9-12900H  2.50 GHz, 32GB RAM, 
on MATLAB R2019b. 
The code and data sets are available on \url{https://gitlab.com/ngillis/check-ssc}. 
The code to is also available in Python. The implementation was kindly done by Subhayan Saha.

\subsection{Synthetic data}  

As explained in Section~\ref{sec:formu}, for a matrix to satisfy the SSC, it requires a certain degree of sparsity. 
Let us generate matrices $H \in \mathbb{R}^{r \times n}_+$ whose columns are $k$-sparse, that is, they have $k$ non-zero entries for $1 \leq k \leq r-1$. 
The position of the $k$ non-zero entries are picked uniformly at random, while the $k$ non-zero  values are picked using the uniform distribution in the probability simplex of dimension $k$, via the Dirichlet distribution with all parameters equal to one. 
We will use two values of $n$: $5r$ and $10r$. Note that the SSC is more likely to be satisfied when $n$ is larger, since the cone generated by the columns of $H$ is more likely to be larger. 
Figure~\ref{fig:SSC5n} displays the number of times, over 20 runs, the SSC was satisfied for $k$-sparse $r$-by-$n$ matrices. 
\begin{figure}[ht!]
\begin{center}
\begin{tabular}{cc}
   \includegraphics[height=8cm]{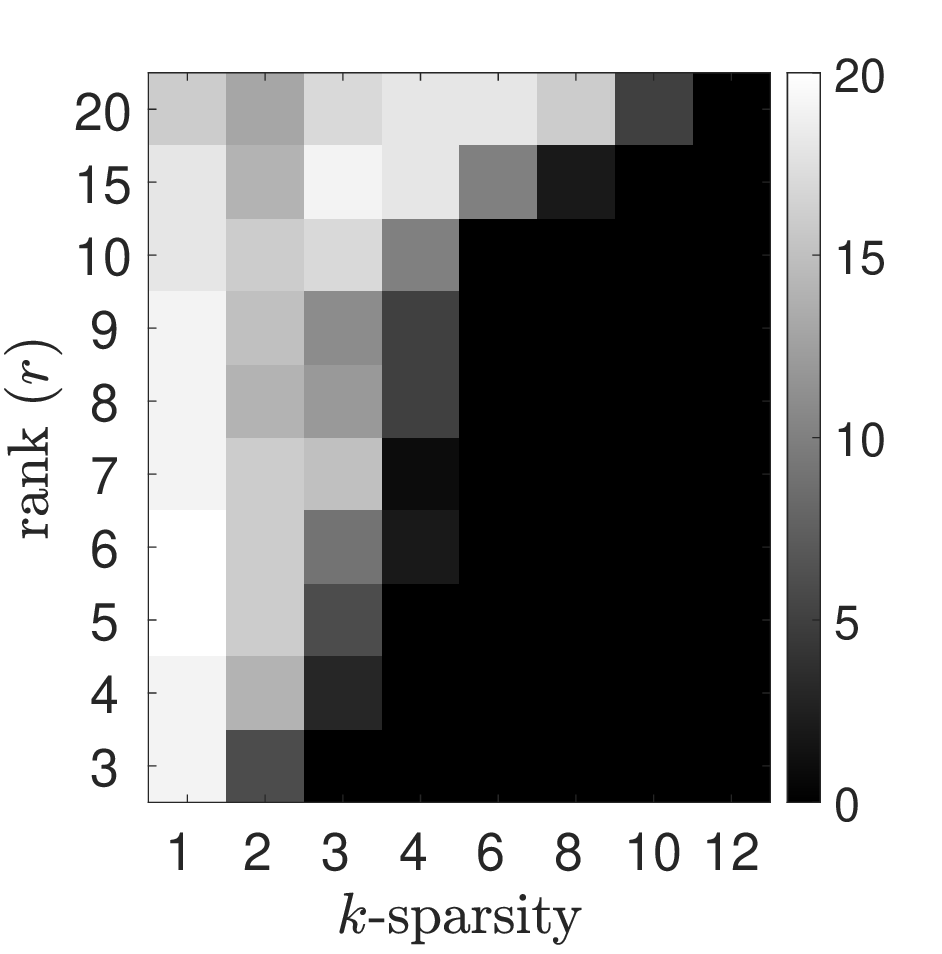}   
   &  \includegraphics[height=8cm]{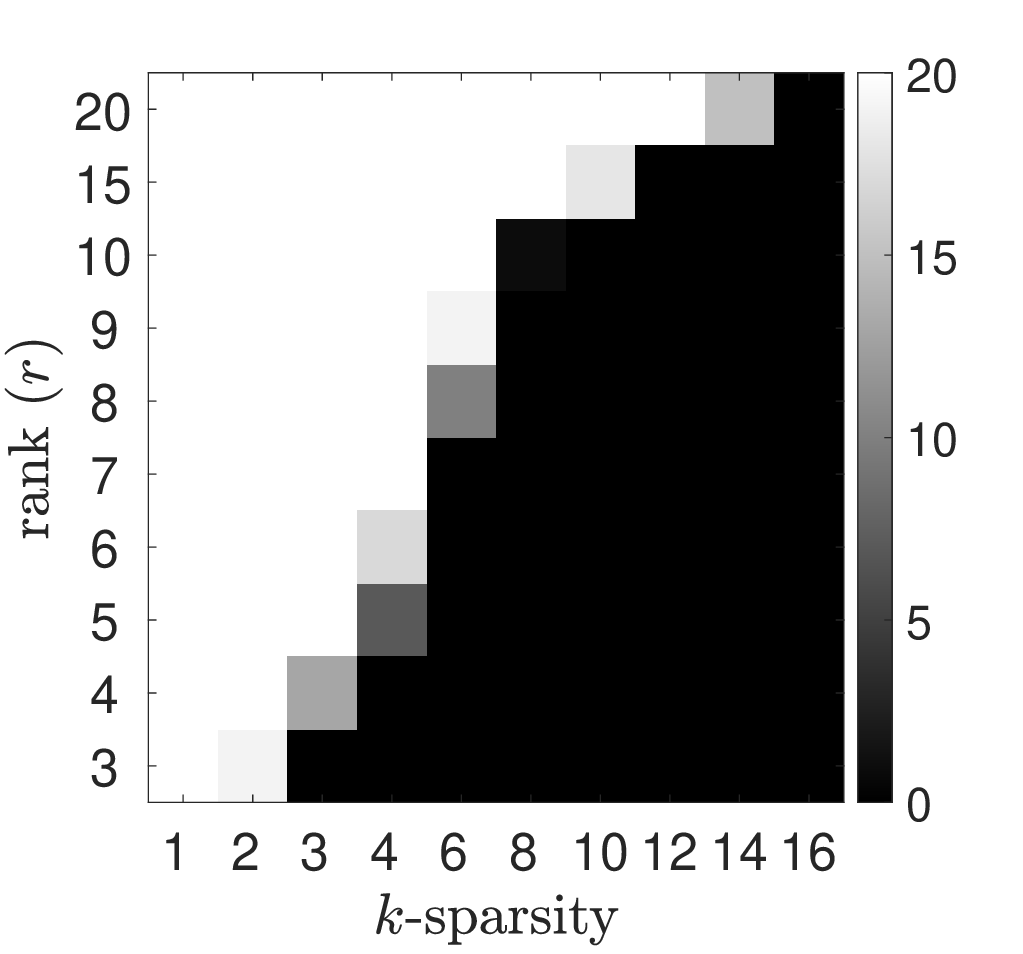} \\ 
  (a) $n=5r$. & (b) $n=10r$. 
\end{tabular} 
\caption{Number of times, over 20 trials, the SSC for $r$-by-$n$ matrices whose columns are $k$-sparse satisfied the SSC.}
\label{fig:SSC5n}
\end{center}
\end{figure} 

 As expected, we observe that, for $n$ larger ($n = 10r$), more matrices satisfy the SSC. In fact, for $n=5r$, many matrices do not satisfy the SSC, even when $k$ is small, because $n$ is not sufficiently large. In particular, even when $k=1$, all matrices do not satisfy the SSC: the reason is that matrices with 1-sparse columns satisfy the SSC if and only if they contain all the unit vectors, up to scaling. Since the columns of $H$ are generated randomly and there are only $5r$ of them, there is a positive and non-negligible probability that not all unitary vectors are generated. 
 In summary, as $n$ increases, the phase transition, that is, the largest value of $k$ that allows the SSC to be satisfied becomes larger.

Tables~\ref{tab:SSCsynt5n} and~\ref{tab:SSCsynt10n} report, for $n = 5r$ and $n=10r$ respectively, the average computational time, the number of instances that reached the 5-minute time limit, and the number of times the NC-SSC was satisfied but not the SSC. 
\begin{table}[h!]
    \centering
    \begin{tabular}{c||c|c|c|c|c|c|c|c}
$r$/$k$ &  1 &  2 &  3 &  4 &  6 &  8 & 10 & 12  \\ \hline  \hline  
 3 & 0.0 & 0.0,  0,  2&         /      &         /      &         /      &         /      &         /      &         /      \\ \hline 
 4 & 0.0 & 0.0,  0,  1& 0.0,  0,  1&         /      &         /      &         /      &         /      &         /      \\ \hline 
 5 & 0.0 & 0.0,  0,  1& 0.1,  0,  4& 0.0 &         /      &         /      &         /      &         /      \\ \hline 
 6 & 0.0 & 0.1 & 0.1,  0,  4& 0.0,  0,  3&         /      &         /      &         /      &         /      \\ \hline 
 7 & 0.0 & 0.1 & 0.3,  0,  1& 0.0,  0,  2& 0.0 &         /      &         /      &         /      \\ \hline 
 8 & 0.1 & 0.1,  0,  2& 0.3,  0,  4& 0.3,  0,  7& 0.0 &         /      &         /      &         /      \\ \hline 
 9 & 0.0 & 0.2,  0,  3& 0.4,  0,  2& 0.4,  0,  7& 0.0 & 0.0 &         /      &         /      \\ \hline 
10 & 0.1 & 0.2,  0,  2& 0.9,  0,  3& 1.1,  0,  4& 0.0,  0,  1& 0.0 &         /      &         /      \\ \hline 
15 & 0.1 & 1.1,  0,  3& 7.1,  0,  1& 280 & 117,  0,  4& 30,  2,  2& 0.0 & 0.0 \\ \hline 
20 & 0.2 & 3.0,  0,  1& 72,  0,  1& 271, 18  &  271, 18,  2& 240, 16,  4& 90,  6,  5& 0.0  
    \end{tabular}
    \caption{Checking the SSC for  $r$-by-$5r$  matrices whose columns are $k$-sparse. The table reports, over 20 trials, the average computational time in seconds, and, if they are non-zeros, the number of times the 5-minute time limit was reached, and the number of times the NC-SSC was satisfied but not the SSC. To make the table compact, zeros are not reported after the computational time. For example, 26 for $r=15$, $k=4$ means (26,0,0), that is, no run went over the 5 minutes, and all matrices that satisfied the NC-SSC satisfied the SSC. Similarly, (293, 16) for $r=20$, $k=4$ means (293,16,0), that is, 16 runs went over the 5 minutes, and all matrices satisfied that satisfied the NC-SSC satisfied the SSC.} 
    \label{tab:SSCsynt5n}
\end{table}
\begin{table}[h!]
    \centering
    \begin{tabular}{c||c|c|c|c|c|c|c|c|c|c} 
 $r$/$k$ &  1 &  2 &  3 &  4 &  6 &  8 & 10 & 12   & 14 & 16 \\ \hline \hline  
     3 & 0.0  & 0.1  &         /      &         /      &         /      &         /      &         /      &         /      &         /      &         /      \\ \hline 
 4 & 0.0 & 0.1 & 0.1,  0,  1&         /      &         /      &         /      &         /      &         /      &         /      &         /      \\ \hline 
 5 & 0.0 & 0.2 & 0.2 & 0.1,  0,  2&         /      &         /      &         /      &         /      &         /      &         /      \\ \hline 
 6 & 0.0 & 0.3 & 0.4 & 0.5,  0,  1&         /      &         /      &         /      &         /      &         /      &         /      \\ \hline 
 7 & 0.0 & 0.3 & 0.5 & 0.7 & 0.0 &         /      &         /      &         /      &         /      &         /      \\ \hline 
 8 & 0.1 & 0.4 & 0.7 & 1.1 & 1.2 &         /      &         /      &         /      &         /      &         /      \\ \hline 
 9 & 0.0 & 0.4 & 1.0 & 1.7 & 3.5 & 0.0 &         /      &         /      &         /      &         /      \\ \hline 
10 & 0.1 & 0.6 & 1.6 & 2.5 & 6.3 & 0.7,  0,  1&         /      &         /      &         /      &         /      \\ \hline 
15 & 0.1 & 1.7 & 5.8 & 26 & 127 & 301, 20  &  271, 18  &  0.0 & 0.0 &         /      \\ \hline 
20 & 0.2 & 3.2 & 31 & 293, 16  &  301, 20  &  300, 20  &  300, 20  &  300, 20  &  225, 15 & 0.0  
    \end{tabular}
    \caption{Checking the SSC for  $r$-by-$10r$  matrices whose columns are $k$-sparse. The table reports, over 20 trials, the average computational time in seconds, and, if they are non-zeros, the number of times the 5-minute time limit was reached, and the number of times the NC-SSC was satisfied but not the SSC. To make the table compact, zeros are not reported after the computational time, as in Table~\ref{tab:SSCsynt5n}. 
    } 
    \label{tab:SSCsynt10n}
\end{table}

We observe the following: 
\begin{itemize}

\item It is harder for Gurobi to check the SSC close to the phase transition; this explains why the computational cost increases and then decreased as $k$ increases. 

\item For $r \leq 10$ and $n=5r$, the computational time is on average at most 1.1 seconds, and for  $r \leq 10$ and $n=10r$, at most 6.3 seconds.  

\item For $n = 5r$, the time limit is often reached for $r=20$ and $4 \leq k \leq 10$. For $n = 10r$, it happens for $r = 15$ and $k=8, 10$, and for $r=20$ and $4 \leq k \leq 14$.

\item Close to the phase transition, it happens more often that the SSC is not satisfied while the NC-SSC is. This happens for example 7 times out of 20 for $n=5r$, $r=8,9$ and $k=4$. However, this does not happen often when $n=10r$: only 5 times over all the cases.  
    
\end{itemize}

One may be a bit disappointed by the fact that Gurobi reaches the time limit on these medium-scale problems. However, in practice, typically $r$ and $k$ are small. For example, 
\begin{itemize}
    \item in hyperspecral imaging, $n$ is the number of pixels in the images (typically larger than 10000), $r$ is the number of materials present in the image (typically smaller than 10), and 
$k$ is the number of materials present in the pixels (typically smaller than 3). 

\item in topic modeling,  $n$ is the number of documents (typically larger than 1000), $r$ is the number of topics discussed in these documents  (typically smaller than 30), and 
$k$ is the number of topics discussed by the  documents present each pixel (a few). 
\end{itemize}

In these real-world scenarios, because $r$ and $k$ are small, it is very possible that Gurobi can check the SSC within a reasonable amount of time. This is illustrated in the next section on hyperspectral images. 

Moreover, if Gurobi does not succeed to solve~\eqref{eq:SSCoptBND} within the allotted time, it still provides a useful information: it means it was not able to find other solutions than the $\{e_i\}$'s and hence provides a necessary condition for the SSC, stronger than the NC-SSC.

\subsection{Real hyperspectral images} \label{sec:hsi}

Given a matrix $X = W^{\#}H^{\#}$, where $H^{\#}$ satisfies the SSC, $(W^{\#},H^{\#})$ can be identified by solving the following minimum-volume NMF problem~\cite{leplat2019blind}:   
\begin{equation} \label{eq:ExactminvolNMF}
\min_{W, H} \det(W^\top W) \quad \text{ such that } \quad 
X = WH, W \geq 0,  H \geq 0 \text{ and } W^\top e = e. 
\end{equation}
In practice, one has to balance the data fitting term, $\|X-WH\|_F^2$, and the volume regularization\footnote{The regularizer $\logdet(W^\top W)$ has been shown to perform better than $\det(W^\top W)$ in practice~\cite{ang2019algorithms} .}, $\logdet(W^\top W)$, by solving 
\begin{equation} \label{eq:minvolNMF}
\min_{W, H} \|X - WH\|_F^2 + \lambda \logdet(W^\top W) \quad  \text{ such that } \quad H \geq 0, W \geq 0 \text{ and }  W^\top e = e, 
\end{equation}
for some penalty parameter $\lambda > 0$. 

We now apply this model on 5 widely-used hyperspectral images: the $(i,j)$th entry of matrix $X \in \mathbb{R}^{m \times n}$ contains the reflectance of the $j$th pixel at the $i$ wavelength. Each column of $X$ is the spectral signature of a pixel, and each row a vectorized image at a given wavelength. Performing NMF on such a matrix allows one to extract the spectral signatures of the pure materials (called endmembers) as the columns of $W$, and their abundances in each pixel as the rows of $H$; see, e.g.,~\cite{ma2014signal} and the references therein.  

We solve~\eqref{eq:minvolNMF} with the code available from \url{https://gitlab.com/ngillis/nmfbook/-/tree/master/algorithms/min-vol%20NMF}, 
and we use the same parameters as in~\cite{gillis2020nonnegative}. 
Table~\ref{tab:SSChsi} reports the results, while Figure~\ref{fig:abmaps} displays the abundance maps, where each image corresponds to a row of $H$ reshaped as an image. 
\begin{table}[h!]
    \centering
    \begin{tabular}{c|c|c|c|c|c|c|c}
    Data set & $m$ & $n$ & $r$ & sparsity ($H$) & NC-SSC & SSC   & time (s.)   \\ \hline 
 Samson  & 156 & $95 \times 95$   & $3$ &  40\%    &  no   & no &  0.0    \\
 Terrain & 166 & $500 \times 307$   & $4$  & 40\%  & yes   & no   &  2 \\ 
 Jasper ridge  & 188 & $100 \times 100$  & $4$   &  45\%    & yes & yes &  0.5    \\
 Urban  & 162 & $307 \times 307$ & $6$   &  50\%   & yes & yes & 9   \\
San Diego airport & 158 & $400\times 400$ & $7$  &    53\%   & yes & yes &  14     
    \end{tabular}
    \caption{Given a minimum-volume NMF of hyperspectral images, $X \approx WH$ with $\det(W^\top W)$ minimized, 
    this table reports 
    the sparsity of $H$ (that is, percentage of zero entries), 
    whether the NC-SSC and the SSC are satisfied, and the time that it took to check the SSC with Gurobi.} 
    \label{tab:SSChsi}
\end{table}
In all cases, the SSC of the matrix $H$ could be checked within 15 seconds. The Terrain data set is interesting because it does satisfy the NC-SSC, but not the SSC. 

To the best of our knowledge, this is the first time the identifiability of the NMF of real-world hyperspectral images is guaranteed. 
More precisely, the solution $(W^*,H^*)$ obtained by solving~\eqref{eq:minvolNMF} on the noisy $X$ is the unique solution of~\eqref{eq:ExactminvolNMF} when factorizing $X^* = W^*H^*$.


\section{Conclusion} 

In this paper, we have provided a formulation, \eqref{eq:SSCoptBND}, suitable to check the SSC with non-convex quadratic optimization solvers (see Theorem~\ref{th:mainTH}), and we used Gurobi. This allows one to check, a posteriori, whether the solutions to various matrix factorization problems with nonnegativity constraints are essentially unique. 
We illustrated the use of our algorithm on synthetic data sets, as well as to check the uniqueness of minimum-volume NMF solution in hyperspectral images.



\begin{figure}[ht!]
\begin{center}
\begin{tabular}{c}
   Samson \\ 
   \includegraphics[width=5.5cm]{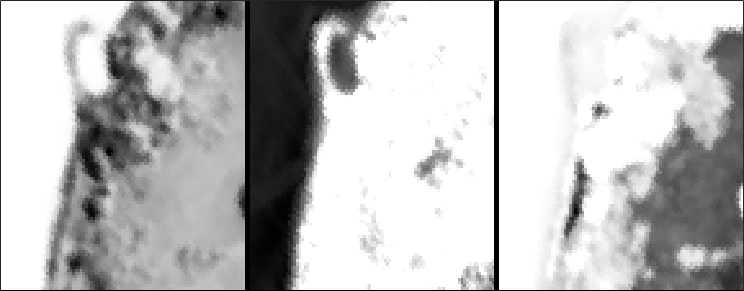}   \\ 
Terrain \\ 
    \includegraphics[width=12cm]{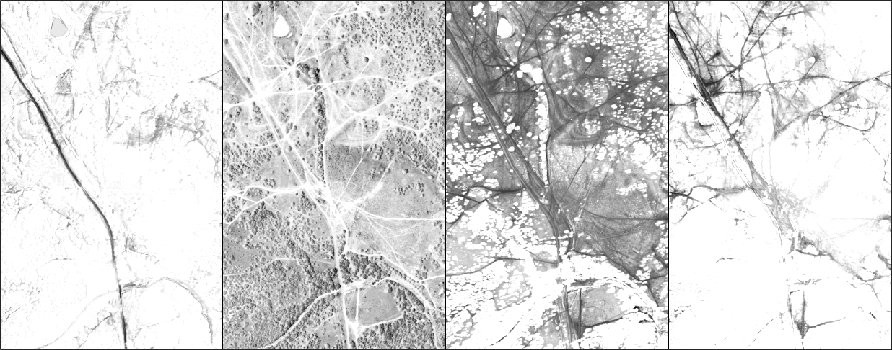} \\ 
    Jasper \\ 
     \includegraphics[width=12cm]{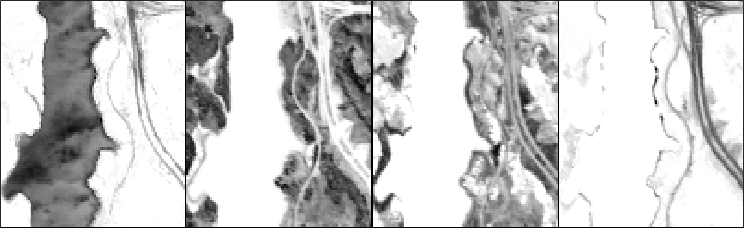} \\ 
     Urban \\ 
      \includegraphics[width=\textwidth]{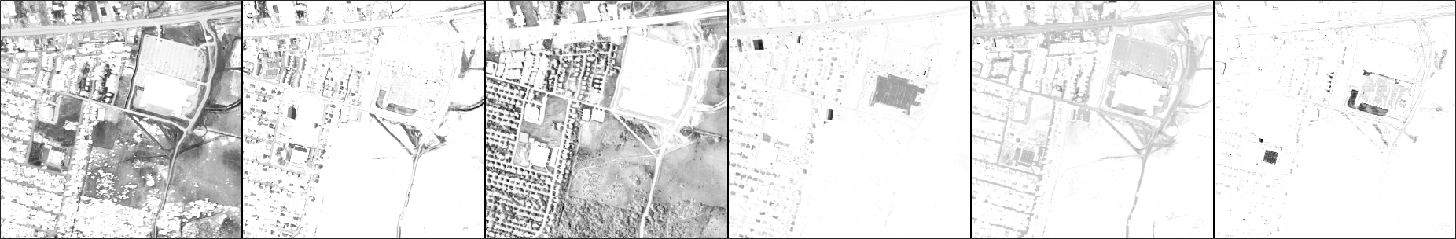} \\ 
      San Diego airport \\
       \includegraphics[width=\textwidth]{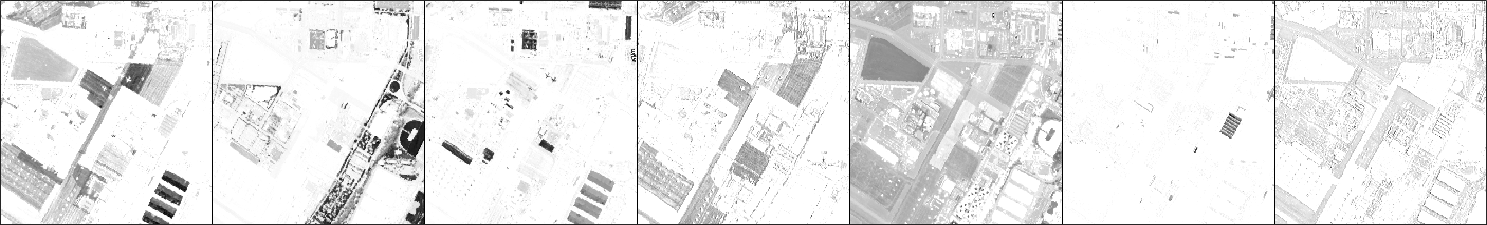}  \\  
\end{tabular} 
\caption{Abundance maps computed via minimum-volume NMF~\eqref{eq:minvolNMF}. Each image corresponds to the abundance map of a material which is a row of $H$ reshaped as an image.}
\label{fig:abmaps}
\end{center}
\end{figure}

\newpage 

\small 
\bibliographystyle{spmpsci}
\bibliography{BibliographyNMFbook}

\end{document}